\documentclass{article}

\usepackage{amsmath}
\usepackage{amsthm}
\usepackage{amssymb}
\usepackage{algorithm}
\usepackage{algpseudocode}
\usepackage{algorithmicx}
\usepackage{booktabs}
\usepackage{caption}
\usepackage{color}
\usepackage{comment}
\usepackage{fancyvrb}
\usepackage{graphicx}
\usepackage{hyperref}
\usepackage{latexsym}
\usepackage{listings}
\usepackage{multicol}
\usepackage{multirow}
\usepackage[normalem]{ulem}
\usepackage{tabularx}
\usepackage{times}
\usepackage{url}
\usepackage[utf8]{inputenc}
\usepackage{xspace}
\usepackage{wrapfig}
\usepackage[]{units}
\usepackage{subfig}
\usepackage{siunitx}
\usepackage{iclr2018_conference}
\usepackage{natbib}
\usepackage{rotating}
\usepackage[export]{adjustbox}
\iclrfinalcopy

\newcommand{\pd}[2]{\partial{#1}/\partial{#2}}
\newcommand{\pdf}[2]{\frac{\partial{#1}}{\partial{#2}}}

\newcommand{\sam}[2]{{#1}^{(#2)}}
\newcommand{\dom}[1]{\mathcal{#1}}

\newcommand{\DD}{F}
\newcommand{\GG}{G}
\newcommand{\GN}{D}
\newcommand{\SN}{T}
\newcommand{\SNp}{\SN_{\dparams}}
\newcommand{\SNo}{\SN^{\star}}
\newcommand{\nonlin}{\nu}

\newcommand{\DV}{\mathcal{D}}

\newcommand{\iw}{w}
\newcommand{\iwo}{\iw^{\star}}
\newcommand{\iwt}{\tilde{\iw}}

\newcommand{\RR}[0]{\mathbb{R}}
\newcommand{\QQ}{\mathbb{Q}}
\newcommand{\PP}{\mathbb{P}}
\newcommand{\EE}{\mathbb{E}}
\newcommand{\ND}[2]{\mathcal{N}(#1, #2)}
\newcommand{\gparams}{\theta}
\newcommand{\dparams}{\phi}
\newcommand{\fd}{f}
\newcommand{\fdc}{{\fd}^{\star}}
\newcommand{\genprob}{\QQ}
\newcommand{\realprob}{\PP}
\newcommand{\gendens}{q}
\newcommand{\realdens}{p}
\newcommand{\genprobp}{\genprob_{\gparams}}
\newcommand{\gendensp}{\gendens_{\gparams}}

\newcommand{\gencondp}[2]{g_{\gparams}(#1 \mid #2)}
\newcommand{\realcondt}[2]{\tilde{p}(#1 \mid #2)}
\newcommand{\prior}{h}
\newcommand{\realdensest}{\tilde{\realdens}}

\newcommand{\dist}[3]{\mathcal{V}(#1, #2, #3)}

\DeclareMathOperator*{\argmax}{\arg \max}
\DeclareMathOperator*{\argmin}{\arg \min}

\newtheorem{theo}{Theorem}

\theoremstyle{definition}

\theoremstyle{definition}
\newtheorem{definition}{Definition}[section]

\title{Boundary-Seeking \\Generative Adversarial Networks}

\author{R Devon Hjelm\thanks{Denotes first-author contributions.} \\
MILA, University of Montr\'eal, IVADO\\
\texttt{erroneus@gmail.com}
\And Athul Paul Jacob\footnotemark[1]\\
MILA, MSR, University of Waterloo\\
\texttt{apjacob@edu.uwaterloo.ca}
\And Tong Che \\
MILA, University of Montr\'eal\\
\texttt{tong.che@umontreal.ca}
\And Adam Trischler \\
MSR\\
\texttt{adam.trischler@microsoft.com}
\And Kyunghyun Cho \\
New York University,\\ CIFAR Azrieli Global Scholar \\
\texttt{kyunghyun.cho@nyu.edu}
\And Yoshua Bengio \\
MILA, University of Montr\'eal, CIFAR, IVADO\\
\texttt{yoshua.bengio@umontreal.ca}
}

\begin{document}

\maketitle

\begin{abstract}
Generative adversarial networks~\citep[GANs, ][]{goodfellow2014generative} are a learning framework that rely on training a discriminator to estimate a measure of difference between a target and generated distributions.
GANs, as normally formulated, rely on the generated samples being completely differentiable w.r.t. the generative parameters, and thus do not work for discrete data.
We introduce a method for training GANs with discrete data that uses the estimated difference measure from the discriminator to compute importance weights for generated samples, thus providing a policy gradient for training the generator.
The importance weights have a strong connection to the decision boundary of the discriminator, and we call our method \emph{boundary-seeking GANs} (BGANs).
We demonstrate the effectiveness of the proposed algorithm with discrete image and character-based natural language generation. 
In addition, the boundary-seeking objective extends to continuous data, which can be used to improve stability of training, and we demonstrate this on Celeba, Large-scale Scene Understanding (LSUN) bedrooms, and Imagenet without conditioning.
\end{abstract}

\section{Introduction}
Generative adversarial networks~\citep[GAN, ][]{goodfellow2014generative} involve a unique generative learning framework that uses two separate models, a generator and discriminator, with opposing or \emph{adversarial} objectives.
Training a GAN only requires back-propagating a learning signal that originates from a learned objective function, which corresponds to the loss of the discriminator trained in an adversarial manner.
This framework is powerful because it trains a generator without relying on an explicit formulation of the probability density, using only samples from the generator to train. 

GANs have been shown to generate often-diverse and realistic samples even when trained on high-dimensional large-scale continuous data~\citep{radford2015unsupervised}.
GANs however have a serious limitation on the type of variables they can model, because they require the composition of the generator and discriminator to be fully differentiable.

With discrete variables, this is not true. For instance, consider using a step function at the end of a generator in order to generate a discrete value. 
In this case, back-propagation alone cannot provide the training signal, because the derivative of a step function is 0 almost everywhere. 
This is problematic, as many important real-world datasets are discrete, such as character- or word-based representations of language.
The general issue of credit assignment for computational graphs with discrete operations (e.g. discrete stochastic neurons) is difficult and open problem, and only approximate solutions have been proposed in the past~\citep{bengio2013estimating,gu2015muprop,gumbel1954statistical,jang2016categorical, maddison2016concrete, tucker2017rebar}.
However, none of these have yet been shown to work with GANs.
In this work, we make the following contributions:
\begin{itemize}
\item 
    We provide a theoretical foundation for \emph{boundary-seeking GANs} (BGAN), a principled method for training a generator of discrete data using a discriminator optimized to estimate an $f$-divergence~\citep{nguyen2010estimating, nowozin2016f}.
    The discriminator can then be used to formulate \emph{importance weights} which provide policy gradients for the generator.
    \item We verify this approach quantitatively works across a set of $\fd$-divergences on a simple classification task and on a variety of image and natural language benchmarks.
    \item We demonstrate that BGAN performs quantitatively better than WGAN-GP~\citep{gulrajani2017improved} in the simple discrete setting.
    \item We show that the boundary-seeking objective extends theoretically to the continuous case and verify it works well with some common and difficult image benchmarks.
    Finally, we show that this objective has some improved stability properties within  training and without.
\end{itemize}

\section{Boundary-Seeking GANs}
In this section, we will introduce boundary-seeking GANs (BGAN), an approach for training a generative model adversarially with discrete data, as well as provide its theoretical foundation.
For BGAN, we assume the normal generative adversarial learning setting commonly found in work on GANs~\citep{goodfellow2014generative}, but these ideas should extend elsewhere.

\subsection{Generative adversarial learning and problem statement}
Assume that we are given empirical samples from a target distribution, $\{\sam{x}{i} \in \dom{X}\}_{i=1}^M$, where $\dom{X}$ is the domain (such as the space of images, word- or character- based representations of natural language, etc.).
Given a random variable $Z$ over a space $\dom{Z}$ (such as $[0, 1]^m$), we wish to find the optimal parameters, $\hat{\gparams} \in \RR^d$, of a function, $\GG_{\gparams}: \dom{Z} \rightarrow \dom{X}$ (such as a deep neural network), whose induced probability distribution, $\genprob_{\gparams}$, describes well the empirical samples.

In order to put this more succinctly, it is beneficial to talk about a probability distribution of the empirical samples, $\realprob$, that is defined on the same space as $\genprob_{\gparams}$.
We can now consider the \emph{difference measure} between $\realprob$ and $\genprob_{\gparams}$, $D(\realprob$, $\genprob_{\gparams})$, so the problem can be formulated as finding the parameters:
\begin{align}
\hat{\gparams} = \argmin_{\gparams} D(\realprob, \genprob_{\gparams}).
\end{align}
Defining an appropriate difference measure is a long-running problem in machine learning and statistics, and choosing the best one depends on the specific setting.
Here, we wish to avoid making strong assumptions on the exact forms of $\realprob$ or  $\genprob_{\gparams}$, and we desire a solution that is scalable and works with very high dimensional data.
Generative adversarial networks~\citep[GANs,][]{goodfellow2014generative} fulfill these criteria by introducing a \emph{discriminator function}, $\GN_{\dparams} : \dom{X} \rightarrow \RR$, with parameters, $\dparams$, then defining a \emph{value function}, 
\begin{align}
\dist{\realprob}{\genprob_{\gparams}}{\GN_{\dparams}} = \EE_{\PP} \left[\log \GN_{\dparams}(x) \right] + \EE_{\prior(z)} \left[\log(1 - \GN_{\dparams}(G(z)) \right],
\end{align}
where samples $z$ are drawn from a simple prior, $\prior(z)$ (such as $U(0, 1)$ or $\ND{0}{1}$).
Here, $\GN_{\dparams}$ is a neural network with a sigmoid output activation, and as such can be interpreted as a simple binary classifier, and the value function can be interpreted as the negative of the \emph{Bayes risk}.
GANs train the discriminator to maximize this value function (minimize the mis-classification rate of samples coming from $\realprob$ or $\genprob_{\gparams}$), while the generator is trained to minimize it. In other words, GANs solve an optimization problem:
\begin{align}
    (\hat{\gparams}, \hat{\dparams}) = \argmin_{\gparams} \argmax_{\dparams} \dist{\realprob}{\genprob_{\gparams}}{\GN_{\dparams}}.
\end{align}
Optimization using only back-propogation and stochastic gradient descent is possible when the generated samples are completely differentiable w.r.t. the parameters of the generator, $\gparams$.

In the non-parametric limit of an optimal discriminator, the value function is equal to a scaled and shifted version of the Jensen-Shannon divergence, $2 * \DV_{JSD}(\realprob || \genprob_{\gparams}) - \log{4}$,\footnote{Note that this has an absolute minimum, so that the above optimization is a Nash-equilibrium} which implies the generator is minimizing this divergence in this limit. 
$f$-GAN~\citep{nowozin2016f} generalized this idea over all $\fd$-divergences, which includes the Jensen-Shannon (and hence also GANs) but also the Kullback–Leibler, Pearson $\chi^2$, and squared-Hellinger.
Their work provides a nice formalism for talking about GANs that use $\fd$-divergences,  which we rely on here.

\begin{definition}[$\fd$-divergence and its dual formulation]
    Let $\fd: \RR_{+} \rightarrow \RR$ be a convex lower semi-continuous function and $\fdc: \dom{C} \subseteq \RR \rightarrow \RR$ be the convex conjugate with domain $\dom{C}$.
    Next, let $\dom{\SN}$ be an arbitrary family of functions, $\dom{\SN} = \{T : \dom{X} \rightarrow \dom{C}\}$.
    Finally, let $\PP$ and $\QQ$ be distributions that are completely differentiable w.r.t. the same Lebesgue measure, $\mu$.\footnote{$\mu$ can be thought of in this context as $x$, so that it can be said that $\PP$ and $\QQ$ have density functions on $x$.}
    The $\fd$-divergence, $\DV_{\fd}(\realprob || \genprob_{\gparams})$, generated by $\fd$, is bounded from below by its dual representation~\citep{nguyen2010estimating},
    \begin{align}
        \DV_{\fd}(\PP || \QQ) = \EE_{\QQ}\left[f \left( \frac{d\PP / d\mu}{d\QQ / d\mu} \right) \right] 
        \ge \sup_{\SN \in \dom{\SN}} (\EE_{\realprob}[\SN(x)] - \EE_{\genprob}[\fdc(\SN(x))]).
        \label{eq:fdiv}
    \end{align}
    \label{def:fdiv}
\end{definition}
The inequality becomes \emph{tight} when $\dom{T}$ is the family of all possible functions.
The dual form allows us to change a problem involving likelihood ratios (which may be intractable) to an maximization problem over $\dom{T}$.
This sort of optimization is well-studied if $\dom{T}$ is a family of neural networks with parameters $\dparams$ (a.k.a., \emph{deep learning}), so the supremum can be found with gradient ascent~\citep{nowozin2016f}.
\begin{definition}[Variational lower-bound for the $\fd$-divergence]
Let $\SN_{\dparams} = \nonlin \circ \DD_{\dparams}$ be a function, which is the composition of an activation function, $\nonlin : \RR \rightarrow \dom{C}$ and a neural network, $\DD_{\dparams} : \dom{X} \rightarrow \RR$.
We can write the \emph{variational lower-bound} of the supremum in Equation~\ref{eq:fdiv} as~\footnote{It can be easily verified that, for $\nonlin(y) = -\log{(1 + e^{-y})}$, $f(u) = u \log{u} + (1 + u) \log{(1 + u)}$, and setting $T = \log{D}$, the variational lower-bound becomes exactly equal to the GAN value function.}:
\begin{align}
    \DV_{\fd}(\PP || \QQ_{\gparams}) \geq \EE_{\realprob}[\nonlin \circ \DD_{\dparams}(x)] - \EE_{\genprob_{\gparams}}[\fdc(\nonlin \circ \DD_{\dparams}(x))] = \dist{\realprob}{\genprob_{\gparams}}{\SN_{\dparams}}.
    \label{eq:fdiv_lower_bound}
\end{align}
\label{def:vlb}
\end{definition}
Maximizing Equation~\ref{eq:fdiv_lower_bound} provides a neural estimator of $f$-divergence, or \emph{neural divergence}~\citep[][]{huang2018parametric}.
Given the family of neural networks, $\dom{T}_{\Phi} = \{T_{\phi}\}_{\phi \in \Phi}$, is sufficiently expressive, this bound can become arbitrarily tight, and the neural divergence becomes arbitrarily close to the true divergence.
As such, GANs are extremely powerful for training a generator of continuous data, leveraging a dual representation along with a neural network with theoretically unlimited capacity to estimate a difference measure.

For the remainder of this work, we will refer to $\SN_{\dparams} = \nonlin \circ \DD_{\dparams}$ as the \emph{discriminator} and $\DD_{\dparams}$ as the \emph{statistic network} (which is a slight deviation from other works).
We use the general term \emph{GAN} to refer to all models that simultaneously minimize and maximize a \emph{variational lower-bound}, $\dist{\realprob}{\genprob_{\gparams}}{\SN_{\dparams}}$, of a difference measure (such as a divergence or distance).
In principle, this extends to variants of GANs which are based on integral probability metrics~\citep[IPMs, ][]{sriperumbudur2009integral} that leverage a dual representation, such as those that rely on restricting $\dom{\SN}$ through parameteric regularization~\citep{arjovsky2017wasserstein} or by constraining its output distribution~\citep{mroueh2017fisher, mroueh2017mcgan, sutherland2016generative}.

\subsection{Estimation of the target distribution}
Here we will show that, with the variational lower-bound of an $\fd$-divergence along with a family of positive activation functions, $\nonlin: \RR \rightarrow \RR_{+}$, we can estimate the target distribution, $\realprob$, using the generated distribution, $\genprob_{\gparams}$, and the discriminator, $\SN_{\dparams}$.

\begin{theo}
Let $f$ be a convex function and $\SNo \in \dom{\SN}$ a function that satisfies the supremum in Equation~\ref{eq:fdiv} in the non-parametric limit.
Let us assume that $\realprob$ and $\genprob_{\gparams}(x)$ are absolutely continuous w.r.t. a measure $\mu$ and hence admit densities, $\realdens(x)$ and $\gendens_{\gparams}(x)$. 
Then the target density function, $\realdens(x)$, is equal to $(\pd{\fdc}{\SN})(\SN^{\star}(x)) \gendens_{\gparams}(x)$.
\label{th:weighted_density}
\end{theo}
\begin{proof}
Following the definition of the $\fd$-divergence and the convex conjugate, we have:
\begin{align}
    \DV_{\fd}(\realprob || \genprob_{\gparams}) = \EE_{\genprob_{\gparams}}\left[\fd\left(\frac{\realdens(x)}{\gendens(x)}\right)\right] 
    = \EE_{\genprob_{\gparams}}\left[\sup_t \left\{t \frac{\realdens(x)}{\gendens(x)} - \fdc(t)\right\}\right].
\end{align}
As $\fdc$ is convex, there is an absolute maximum when $\pdf{\fdc}{t}(t) = \frac{\realdens(x)}{\gendens_{\gparams}(x)}$. Rephrasing $t$ as a function, $\SN(x)$, and by the definition of $\SNo(x)$, we arrive at the desired result.
\end{proof}

Theorem~\ref{th:weighted_density} indicates that the target density function can be re-written in terms of a generated density function and a scaling factor.
We refer to this scaling factor, $\iwo(x) = (\pd{\fdc}{\SN})(\SN^{\star}(x))$, as the optimal \emph{importance weight} to make the connection to importance sampling~\footnote{
In the case of the $\fd$-divergence used in \citet{goodfellow2014generative}, the optimal importance weight equals $\iwo(x) = e^{\DD^{\star}(x)} = \GN^{\star}(x) / (1 - \GN^{\star}(x))$}.
In general, an optimal discriminator is hard to guarantee in the saddle-point optimization process, so in practice, $\SNp$ will define a lower-bound that is not exactly tight w.r.t. the $\fd$-divergence.
Nonetheless, we can define an estimator for the target density function using a sub-optimal $\SN_{\dparams}$.

\begin{definition}[$\fd$-divergence importance weight estimator]
Let $\fd$ and $\fdc$, and $\SN_{\dparams}(x)$ be defined as in Definitions~\ref{def:fdiv} and \ref{def:vlb} but where $\nonlin: \RR \rightarrow \RR_{+} \subseteq \dom{C}$ is a positive activation function.
Let $\iw(x) = (\pd{\fdc}{\SN})(\SN(x))$ and $\beta = \EE_{\genprob_{\dparams}}[\iw(x)]$ be a partition function.
The $\fd$-divergence importance weight estimator, $\realdensest(x)$ is
\begin{align}
    \realdensest(x) = \frac{w(x)}{\beta} \gendens_{\gparams}(x).
    \label{eq:bgan_global}
\end{align}
\end{definition}

\begin{table}[t]
\caption{Important weights and nonlinearities that ensure }
\centering
    \begin{tabular}{ |p{3cm}||p{3cm}|p{4cm}|}
         \hline
         \multicolumn{3}{|c|}{Importance weights for $\fd$-divergences} \\
         \hline
         $\fd$-divergence   & $\nonlin(y)$                  & $\iw(x) = (\pd{\fdc}{\SN})(\SN(x))$\\
         \hline
         GAN                & $-\log{(1 + e^{-y})}$         & $-\frac{1}{1 - e^{-\SN_{\dparams}}} = e^{\DD_{\dparams}(x)}$\\
          \hline
         Jensen-Shannon &   $ \log{2} -\log{(1 + e^{-y})}$  & $-\frac{1}{2 - e^{-\SN_{\dparams}}} = e^{\DD_{\dparams}(x)}$\\
          \hline
         KL                 & $y + 1$                       & $e^{(\SN_{\dparams}(x) - 1)} = e^{\DD_{\dparams}(x)}$\\
          \hline
         Reverse KL         & $-e^{-y}$                     & $-\frac{1}{\SN_{\dparams}(x)} = e^{\DD_{\dparams}(x)}$\\
          \hline
         Squared-Hellinger  & $1 - e^{-v / 2}$              & $\frac{1}{(1 - \SN_{\dparams}(x))^2} = e^{\DD_{\dparams}(x)}$\\
         \hline
    \end{tabular}
    \label{tb:iw_fdiv}
\end{table}

The non-negativity of $\nonlin$ is important as the densities are positive.
Table~\ref{tb:iw_fdiv} provides a set of $\fd$-divergences (following suggestions of ~\citet{nowozin2016f} with only slight modifications) which are suitable candidates and yield positive importance weights.
Surprisingly, each of these yield the same function over the neural network before the activation function: $\iw(x) = e^{\DD_{\dparams}(x)}$.\footnote{Note also that the normalized weights resemble softmax probabilities}
It should be noted that $\realdensest(x)$ is a potentially biased estimator for the true density; however, the bias only depends on the tightness of the variational lower-bound: the tighter the bound, the lower the bias.
This problem reiterates the problem with all GANs, where proofs of convergence are only provided in the optimal or near-optimal limit~\citep{goodfellow2014generative, nowozin2016f, mao2016least}.

\subsection{Boundary-seeking GANs}
As mentioned above and repeated here, GANs only work \emph{when the value function is completely differentiable w.r.t. the parameters of the generator, $\gparams$}.
The gradients that would otherwise be used to train the generator of discrete variables are zero almost everywhere, so it is impossible to train the generator directly using the value function.
Approximations for the back-propagated signal exist~\citep{bengio2013estimating,gu2015muprop,gumbel1954statistical,jang2016categorical,maddison2016concrete,tucker2017rebar}, but as of this writing, none has been shown to work satisfactorily in training GANs with discrete data. 

Here, we introduce the boundary-seeking GAN as a method for training GANs with discrete data.
We first introduce a policy gradient based on the KL-divergence which uses the importance weights as a reward signal.
We then introduce a lower-variance gradient which defines a unique reward signal for each $z$ and prove this can be used to solve our original problem.

\paragraph{Policy gradient based on importance sampling}
Equation~\ref{eq:bgan_global} offers an option for training a generator in an adversarial way.
If we know the explicit density function, $\gendensp$, (such as a multivariate Bernoulli distribution), then we can, using $\realdensest(x)$ as a target (keeping it fixed w.r.t. optimization of $\gparams$), train the generator using the gradient of the KL-divergence:
\begin{align}
    \nabla_{\gparams} \DV_{KL}(\realdensest(x) || \gendens_{\gparams}) = -\EE_{\genprob_{\gparams}}\left[\frac{\iw(x)}{\beta} \nabla_{\gparams} \log{\gendens_{\gparams}(x)}\right].
    \label{eq:bgan_grad}
\end{align}
Here, the connection to importance sampling is even clearer, and this gradient resembles other importance sampling methods for training generative models in the discrete setting~\citep{bornschein2014reweighted,rubinstein2016simulation}.
However, we expect the variance of this estimator will be high, as it requires estimating the partition function, $\beta$ (for instance, using Monte-Carlo sampling).
We address reducing the variance from estimating the normalized importance weights next.

\paragraph{Lower-variance policy gradient}
Let $\gendensp(x) = \int_{\dom{Z}}{\gencondp{x}{z} \prior(z) dz}$ be a probability density function with a conditional density, $\gencondp{x}{z} : \dom{Z} \rightarrow [0, 1]^d$ (e.g., a multivariate Bernoulli distribution), and prior over $z$, $\prior(z)$.
Let $\alpha(z) = \EE_{\gencondp{x}{z}}[\iw(x)] = \int_{\dom{X}}{\gencondp{x}{z} \iw(x) dx}$ be a partition function over the conditional distribution.
Let us define $\realcondt{x}{z} = \frac{w(x)}{\alpha(z)} \gencondp{x}{z}$ as the (normalized) conditional distribution weighted by $\frac{w(x)}{\alpha(z)}$.
The expected conditional KL-divergence over $\prior(z)$ is:
\begin{align}
\EE_{\prior(z)}[\DV_{KL}\left(\realcondt{x}{z} \middle\| \gencondp{x}{z} \right)] = \int_{\dom{Z}}{\prior(z) \DV_{KL}\left(\realcondt{x}{z} \middle\| \gencondp{x}{z} \right) dz}
\label{eq:bgan_local}
\end{align}

Let $\sam{x}{m} \sim \gencondp{x}{z}$ be samples from the prior and $\iwt(\sam{x}{m}) = \frac{\iw(\sam{x}{m})}{\sum_{m'}{\iw(\sam{x}{m'})}}$ be a Monte-Carlo estimate of the normalized importance weights. 
The gradient of the expected conditional KL-divergence w.r.t. the generator parameters, $\gparams$, becomes:
\begin{align}
    \nabla_{\gparams} \EE_{\prior(z)}[\DV_{KL}\left(\realcondt{x}{z} \middle\| \gencondp{x}{z} \right)] = -\EE_{\prior(z)}\left[\sum_m \iwt(\sam{x}{m}) \nabla_{\gparams} \log{\gencondp{\sam{x}{m}}{z}}\right],
    \label{eq:bgan_alpha}
\end{align}
where we have approximated the expectation using the Monte-Carlo estimate.

Minimizing the expected conditional KL-divergences is stricter than minimizing the KL-divergence in Equation~\ref{eq:bgan_global}, as it requires all of the conditional distributions to match independently.
We show that the KL-divergence of the marginal probabilities is zero when the expectation of the conditional KL-divergence is zero as well as show this estimator works better in practice in the Appendix.

\begin{algorithm}[t]
    \begin{algorithmic}
        \State $(\theta, \phi) \gets \text{initialize the parameters of the generator and statistic network}$
        \Repeat
        \State $\sam{\hat{x}}{n} \sim \realprob$
        \Comment{Draw $N$ samples from the empirical distribution}
        \State $\sam{z}{n} \sim \prior(z)$
        \Comment{Draw $N$ samples from the prior distribution}
        \State $\sam{x}{m|n} \sim \gencondp{x}{\sam{z}{n}}$
        \Comment{Draw $M$ samples from each conditional $\gencondp{x}{\sam{z}{m}}$ (drawn independently if $\realprob$ and $\genprobp$ are multi-variate)}
        \State $\iw(\sam{x}{m|n}) \gets (\pd{\fdc}{\SN}) \circ (\nonlin \circ \DD_{\dparams}(\sam{x}{m|n}))$
        \State $\iwt(\sam{x}{m|n}) \gets \iw(\sam{x}{m|n}) / \sum_{m'} \iw(\sam{x}{m'|n})$
        \Comment{Compute the un-normalized and normalized importance weights (applied uniformly if $\realprob$ and $\genprobp$ are multi-variate)}
        \State $\dist{\realprob}{\genprob_{\gparams}}{\SN_{\dparams}} \gets \frac{1}{N} \sum_n \DD_{\dparams}(\sam{\hat{x}}{n}) - \frac{1}{N} \sum_n \frac{1}{M} \sum_m \iw(\sam{x}{m|n})$
        \Comment{Estimate the variational lower-bound}
        \State $\dparams \gets \dparams + \gamma_d \nabla_{\dparams} \dist{\realprob}{\genprob_{\gparams}}{\SN_{\dparams}}$
        \Comment{Optimize the discriminator parameters}
        \State $\gparams \gets \gparams + \gamma_g \frac{1}{N} \sum_{n,m} \iwt(\sam{x}{m \mid n}) \nabla_{\gparams} \log{\gencondp{\sam{x}{m \mid n}}{z}}$
        \Comment{Optimize the generator parameters}
        \Until{convergence}
    \end{algorithmic}
\caption{\label{alg:bgan}. Discrete Boundary Seeking GANs}
\end{algorithm}

Algorithm~\ref{alg:bgan} describes the training procedure for discrete BGAN.
This algorithm requires an additional $M$ times more computation to compute the normalized importance weights, though these can be computed in parallel exchanging space for time.
When the $\realprob$ and $\genprobp$ are multi-variate (such as with discrete image data), we make the assumption that the observed variables are independent conditioned on $Z$.
The importance weights, $\iw$, are then applied \emph{uniformly} across each of the observed variables.

\paragraph{Connection to policy gradients}
REINFORCE is a common technique for dealing with discrete data in GANs~\citep{che2017maximum, li2017adversarial}.
Equation~\ref{eq:bgan_local} is a policy gradient in the special case that the reward is the normalized importance weights.
This reward approaches the likelihood ratio in the non-parametric limit of an optimal discriminator.
Here, we make another connection to REINFORCE as it is commonly used, with baselines, by deriving the gradient of the reversed KL-divergence.
\begin{definition}[REINFORCE-based BGAN]
Let $\SN_{\dparams}(x)$ be defined as above where\\ $\pd{\fdc}{T}(\SN_{\dparams}(x)) = e^{\DD_{\dparams}(x)}$.
Consider the gradient of the \emph{reversed} KL-divergence:
\begin{align}
    \nabla_{\gparams} \DV_{KL}\left(\gendensp \middle\| \realdensest \right) &= -\EE_{\prior(z)}\left[\sum_m (\log{\iw(\sam{x}{m})} - \log{\beta} + 1) \nabla_{\gparams} \log{\gencondp{\sam{x}{m}}{z}}\right]
    \nonumber\\
    &= -\EE_{\prior(z)}\left[\sum_m (\DD_{\dparams}(x) - b) \nabla_{\gparams} \log{\gencondp{\sam{x}{m}}{z}}\right]
\end{align}
\end{definition}
From this, it is clear that we can consider the output of the statistic network, $\DD_{\dparams}(x)$, to be a \emph{reward} and $b = \log{\beta} = \EE_{\genprobp}[\iw(x)]$ to be the analog of a baseline.\footnote{Note that we have removed the additional constant as $\EE_{\gendensp}[1 * \nabla_{\gparams} \gendensp] = 0$}
This gradient is similar to those used in previous works on discrete GANs, which we discuss in more detail in Section~\ref{sec:rel}.

\subsection{Continuous variables and the stability of GANs}
For continuous variables, minimizing the variational lower-bound suffices as an optimization technique as we have the full benefit of back-propagation to train the generator parameters, $\gparams$.
However, while the convergence of the discriminator is straightforward, to our knowledge there is no general proof of convergence for the generator except in the non-parametric limit or near-optimal case.
What's worse is the value function can be arbitrarily large and negative.
Let us assume that $\max \SN = M < \infty$ is unique. As $\fdc$ is convex, the minimum of the lower-bound over $\gparams$ is:
\begin{align}
    &\inf_{\gparams} \dist{\realprob}{\genprob_{\gparams}}{\GN_{\dparams}} 
    = \inf_{\gparams} \EE_{\realprob}[\SN_{\dparams}(x)] - \EE_{\genprob_{\gparams}}[\fdc(\SN_{\dparams}(x))] 
    \nonumber\\
    &= \EE_{\realprob}[\SN_{\dparams}(x)] - \sup_{\gparams} \EE_{\genprob_{\gparams}}[\fdc(\SN_{\dparams}(x))]
    = \EE_{\realprob}[\SN_{\dparams}(x)] - \fdc(M).
\end{align}
In other words, the generator objective is optimal when the generated distribution, $\QQ_{\gparams}$, is nonzero only for the set $\{x \mid \SN(x) = M \}$.
Even outside this worst-case scenario, the additional consequence of this minimization is that this variational lower-bound can become looser w.r.t. the $\fd$-divergence, with no guarantee that the generator would actually improve.
Generally, this is avoided by training the discriminator in conjunction with the generator, possibly for many steps for every generator update.
However, this clearly remains one source of potential instability in GANs.

Equation~\ref{eq:bgan_global} reveals an alternate objective for the generator that should improve stability.
Notably, we observe that for a given estimator, $\realdensest(x)$, $\gendensp(x)$ matches when $\iw(x) = (\pd{\fdc}{\SN})(\SN(x)) = 1$.
\begin{definition}[Continuous BGAN objective for the generator]
Let $\GG_{\gparams} : \dom{Z} \rightarrow \dom{X}$ be a generator function that takes as input a latent variable drawn from a simple prior, $z \sim \prior(z)$.
Let $\SN_{\dparams}$ and $\iw(x)$ be defined as above.
We define the continuous BGAN objective as: $\hat{\gparams} = \argmin_{\gparams} (\log{\iw(\GG_{\gparams}(z))})^2.$
We chose the $\log$, as with our treatments of $\fd$-divergences in Table~\ref{tb:iw_fdiv}, the objective is just the square of the statistic network output:
\begin{align}
\hat{\gparams} = \argmin_{\gparams} \DD_{\dparams}(\GG_{\gparams}(z))^2.
\end{align}
\end{definition}
This objective can be seen as changing a concave optimization problem (which is poor convergence properties) to a convex one.

\section{Related work and discussion}
\label{sec:rel}
\paragraph{On estimating likelihood ratios from the discriminator}
Our work relies on estimating the likelihood ratio from the discriminator, the theoretical foundation of which we draw from $f$-GAN~\citep{nowozin2016f}.
The connection between the likelihood ratios and the policy gradient is known in previous literature~\citep{jie2010connection}, and the connection between the discriminator output and the likelihood ratio was also made in the context of continuous GANs~\citep{mohamed2016learning, tran2017deep}. 
However, our work is the first to successfully formulate and apply this approach to the discrete setting.

\paragraph{Importance sampling}
Our method is very similar to re-weighted wake-sleep~\citep[RWS, ][]{bornschein2014reweighted}, which is a method for training Helmholtz machines with discrete variables.
RWS also relies on minimizing the KL divergence, the gradients of which also involve a policy gradient over the likelihood ratio.
Neural variational inference and learning~\citep[NVIL,][]{mnih2014neural}, on the other hand, relies on the reverse KL.
These two methods are analogous to our importance sampling and REINFORCE-based BGAN formulations above.

\paragraph{GAN for discrete variables}
Training GANs with discrete data is an active and unsolved area of research, particularly with language model data involving recurrent neural network (RNN) generators~\citep{yu2016seqgan, li2017adversarial}.
Many REINFORCE-based methods have been proposed for language modeling~\citep{yu2016seqgan, li2017adversarial, dai2017towards} which are similar to our REINFORCE-based BGAN formulation and effectively use the sigmoid of the estimated log-likelihood ratio.
The primary focus of these works however is on improving \emph{credit assignment}, and their approaches are compatible with the policy gradients provided in our work.

There have also been some improvements recently on training GANs on language data by rephrasing the problem into a GAN over some continuous space~\citep{lamb2016professor, kim2017adversarially, gulrajani2017improved}.
However, each of these works bypass the difficulty of training GANs with discrete data by rephrasing the deterministic game in terms of continuous latent variables or simply ignoring the discrete sampling process altogether, and do not directly solve the problem of optimizing the generator from a difference measure estimated from the discriminator. 


\paragraph{Remarks on stabilizing adversarial learning, IPMs, and regularization}
A number of variants of GANs have been introduced recently to address stability issues with GANs.
Specifically, generated samples tend to collapse to a set of singular values that resemble the data on neither a per-sample or distribution basis.
Several early attempts in modifying the train procedure~\citep{berthelot2017began, salimans2016improved} as well as the identifying of a taxonomy of working architectures~\citep{radford2015unsupervised} addressed stability in some limited setting, but it wasn't until Wassertstein GANs~\citep[WGAN,][]{arjovsky2017wasserstein} were introduced that there was any significant progress on reliable training of GANs.

WGANs rely on an integral probability metric~\citep[IPM,][]{sriperumbudur2009integral} that is the dual to the \emph{Wasserstein distance}.
Other GANs based on IPMs, such as Fisher GAN~\citep{mroueh2017fisher} tout improved stability in training.
In contrast to GANs based on $\fd$-divergences, besides being based on \emph{metrics} that are ``weak", IPMs rely on restricting $\dom{\SN}$ to a subset of all possible functions.
For instance in WGANs, $\dom{\SN} = \{\SN \mid \|\SN\|_{L} \leq K\}$, is the set of K-Lipschitz functions.
Ensuring a statistic network, $\SNp$, with a large number of parameters is Lipschitz-continuous is \emph{hard}, and these methods rely on some sort of regularization to satisfy the necessary constraints.
This includes the original formulation of WGANs, which relied on weight-clipping, and a later work~\citep{gulrajani2017improved} which used a gradient penalty over interpolations between real and generated data.

Unfortunately, the above works provide little details on whether $\SNp$ is actually in the constrained set in practice, as this is probably very hard to evaluate in the high-dimensional setting.
Recently, \citet{roth2017stabilizing} introduced a gradient norm penalty similar to that in \citet{gulrajani2017improved} without interpolations and which is formulated in terms of $\fd$-divergences.
In our work, we've found that this approach greatly improves stability, and we use it in nearly all of our results.
That said, it is still unclear empirically how the discriminator objective plays a strong role in stabilizing adversarial learning, but at this time it appears that correctly regularizing the discriminator is sufficient.

\section{Discrete variables: experiments and results}

\subsection{Adversarial classification}
We first verify the gradient estimator provided by BGAN works quantitatively in the discrete setting by evaluating its ability to train a classifier with the CIFAR-10 dataset~\citep{krizhevsky2009learning}.
The ``generator" in this setting is a multinomial distribution, $\gencondp{y}{x}$ modeled by the softmax output of a neural network.
The discriminator, $\SNp(x, y)$, takes as input an image / label pair so that the variational lower-bound is:
\begin{align}
    \dist{\realprob_{X Y}}{\genprob_{Y|X} \realprob_X}{\SN_{\dparams}} = \EE_{\realdens(x, y)}[\SN_{\dparams}(x, y)] - \EE_{\gencondp{y}{x}\realdens(x)}[\fdc(\SN_{\dparams}(x, y))]
\end{align}
For these experiments, we used a simple $4$-layer convolutional neural network with an additional $3$ fully-connected layers.
We trained the importance sampling BGAN on the set of $\fd$-divergences given in Table~\ref{tb:iw_fdiv} as well as the REINFORCE counterpart for $200$ epochs and report the accuracy on the test set.
In addition, we ran a simple classification baseline trained on cross-entropy as well as a continuous approximation to the problem as used in WGAN-based approaches~\citep{gulrajani2017improved}.
No regularization other than batch normalization~\citep[BN,][]{ioffe2015batch} was used with the generator, while gradient norm penalty~\citep{roth2017stabilizing} was used on the statistic networks.
For WGAN, we used clipping, and chose the clipping parameter, the number of discriminator updates, and the learning rate separately based on training set performance.
The baseline for the REINFORCE method was learned using a moving average of the reward.

\begin{table}[ht]
    \centering
    \caption{Adversarial classification on CIFAR-10.
    All methods are BGAN with importance sampling (left) or REINFORCE (right) except for the baseline (cross-entropy) and Wasserstein GAN (WGAN)}
    \begin{tabular}{ |l | l||c|c|} 
        \hline
        & Measure            & \multicolumn{2}{c|}{Error(\%)} \\ 
        \hline
        \hline
        & Baseline           & \multicolumn{2}{c|}{26.6}   \\
        \hline
        &WGAN (clipping)    & \multicolumn{2}{c|}{72.3}      \\
        \hline
        \hline
        &                   & IS            & REINFORCE \\
        \multirow{4}{*}{\begin{turn}{90}BGAN\end{turn}} 
        &GAN                & 26.2          & 27.1  \\
        &Jensen-Shannon     & 26.0          & 27.7  \\
        &KL                 & 28.1          & 28.0  \\
        &Reverse KL         & 27.8          & 28.2  \\
        &Squared-Hellinger  & 27.0          & 28.0  \\
        \hline
    \end{tabular}
    \label{tb:class}
\end{table}

Our results are summarized in Table~\ref{tb:class}.
Overall, BGAN performed similarly to the baseline on the test set, with the REINFORCE method performing only slightly worse.
For WGAN, despite our best efforts, we could only achieve an error rate of $72.3\%$ on the test set, and this was after a total of $600$ epochs to train.
Our efforts to train WGAN using gradient penalty failed completely, despite it working with higher-dimension discrete data (see Appendix).

\subsection{Discrete image and natural language generation}
\label{sec:disc_results}
\paragraph{Image data: binary MNIST and quantized CelebA}

We tested BGAN using two imaging benchmarks: the common discretized MNIST dataset~\citep{salakhutdinov2008quantitative} and a new quantized version of the CelebA dataset \citep[see ][for the original CelebA dataset]{liu2015deep}.

For CelebA quantization, we first downsampled the images from $64\times64$ to $32\times32$.
We then generated a $16$-color palette using Pillow, a fork of the Python Imaging Project (\href{https://python-pillow.org}{https://python-pillow.org}).
This palette was then used to quantize the RGB values of the CelebA samples to a one-hot representation of $16$ colors.
Our models used deep convolutional GANs~\citep[DCGAN,][]{radford2015unsupervised}.
The generator is fed a vector of $64$ i.i.d. random variables drawn from a uniform distribution, $[0, 1]$.
The output nonlinearity was sigmoid for MNIST to model the Bernoulli centers for each pixel, while the output was softmax for quantized CelebA.

\begin{figure*}[bt]
    \begin{minipage}{0.32\textwidth}
        \includegraphics[width=0.99\columnwidth]{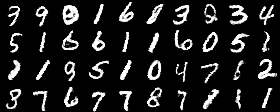}
    \end{minipage}
    \hfill
    \begin{minipage}{0.66\textwidth}
        \caption{
        Left: Random samples from the generator trained as a boundary-seeking GAN (BGAN) with discrete MNIST data.
        Shown are the Bernoulli centers of the generator conditional distribution.}
        \label{fig:rand_gen}
    \end{minipage}

    \begin{minipage}{0.32\textwidth}
        \includegraphics[width=0.99\columnwidth]{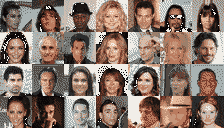}
    \end{minipage}
    \hfill
    \begin{minipage}{0.32\textwidth}
        \includegraphics[width=0.99\columnwidth]{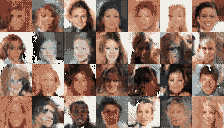}
    \end{minipage}
    \hfill
    \begin{minipage}{0.32\textwidth}
        \caption{
        Left: Ground-truth 16-color (4-bit) quantized CelebA images downsampled to $32\times32$.
        Right: Samples produced from the generator trained as a boundary-seeking GAN on the quantized CelebA for $50$ epochs.}
        \label{fig:celeb_gt}
    \end{minipage}
    \vspace{-3mm}
\end{figure*}

Our results show that training the importance-weighted BGAN on discrete MNIST data is stable and produces realistic and highly variable generated handwritten digits (Figure~\ref{fig:rand_gen}).
Further quantitative experiments comparing BGAN against WGAN with the gradient penalty~\citep[WGAN-GP][]{gulrajani2017improved} showed that when training a new discriminator on the samples directly (keeping the generator fixed), the final estimated distance measures were \emph{higher} (i.e., worse) for WGAN-GP than BGAN, \emph{even when comparing using the Wasserstein distance}.
The complete experiment and results are provided in the Appendix.
For quantized CelebA, the generator trained as a BGAN produced reasonably realistic images which resemble the original dataset well and with good diversity.

\paragraph{1-billion word}
Next, we test BGAN in a natural language setting with the 1-billion word dataset~\citep{chelba2013one}, modeling at the character-level and limiting the dataset to sentences of at least $32$ and truncating to $32$ characters.
For character-level language generation, we follow the architecture of recent work~\citep{gulrajani2017improved}, and use deep convolutional neural networks for both the generator and discriminator.

\begin{table}
\caption{Random samples drawn from a generator trained with the discrete BGAN objective.
The model is able to successfully learn many important character-level English language patterns.}
\small
\begin{tabular}{c | c | c}
And it 's miant a quert could he 
& He weirst placed produces hopesi 
& What 's word your changerg bette\\
" We pait of condels of money wi 
& Sance Jory Chorotic , Sen doesin
& In Lep Edger 's begins of a find", \\
Lankard Avaloma was Mr. Palin , 
& What was like one of the July 2 
& " I stroke like we all call on a\\
Thene says the sounded Sunday in 
& The BBC nothing overton and slea
& With there was a passes ipposing\\
About dose and warthestrinds fro 
& College is out in contesting rev
& And tear he jumped by even a roy
\end{tabular}
\label{tab:lm_res}
\end{table}

Training with BGAN yielded stable, reliably good character-level generation (Table~\ref{tab:lm_res}), though generation is poor compared to recurrent neural network-based methods~\citep{sutskever2011generating,mikolov2012statistical}.
However, we are not aware of any previous work in which a discrete GAN, without any continuous relaxation~\citep{gulrajani2017improved}, was successfully trained from scratch without pretraining and without an auxiliary supervised loss to generate any sensible text. Despite the low quality of the text relative to supervised recurrent language models, the result demonstrates the stability and capability of the proposed boundary-seeking criterion for training discrete GANs.

\section{Continuous variables: experiments and results}
Here we present results for training the generator on the boundary-seeking objective function.
In these experiments, we use the original GAN variational lower-bound from ~\citet{goodfellow2014generative}, only modifying the generator function.
All results use gradient norm regularization~\citep{roth2017stabilizing} to ensure stability.

\subsection{Generation benchmarks}
We test here the ability of continuous BGAN to train on high-dimensional data.
In these experiments, we train on the CelebA, LSUN~\citep{yu15lsun} datasets, and the 2012 ImageNet dataset with all $1000$ labels~\citep{krizhevsky2012imagenet}.
The discriminator and generator were both modeled as $4$-layer Resnets~\citep{he2016deep} without conditioning on labels or attributes.

\begin{figure*}[t]
\begin{minipage}{0.49\textwidth}
\centering
\includegraphics[width=0.99\columnwidth]{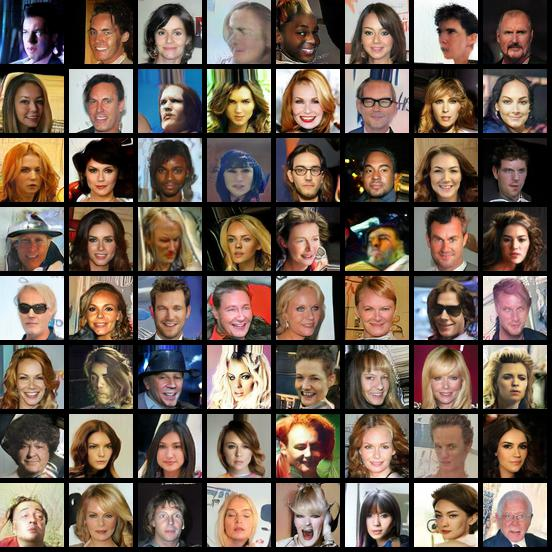}

CelebA
\end{minipage}
\hfill
\begin{minipage}{0.49\textwidth}
\centering
\includegraphics[width=0.99\columnwidth]{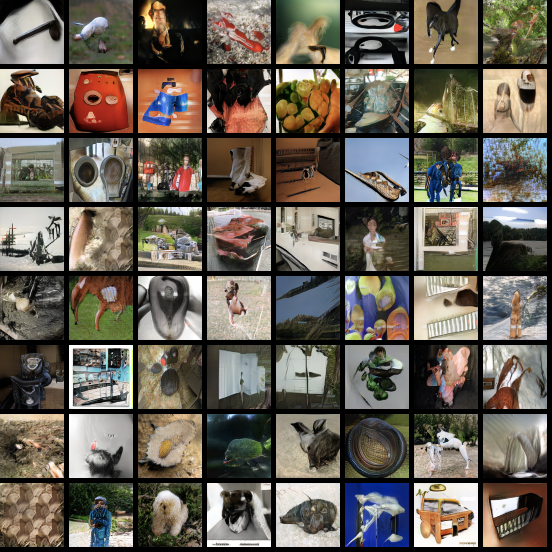}

Imagenet
\end{minipage}

\begin{minipage}{0.49\textwidth}
\centering
\adjincludegraphics[width=0.99\columnwidth,trim={0 0 0 {.5\width}},clip]{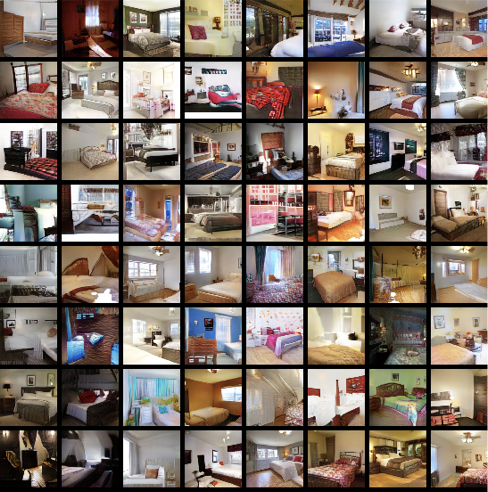}

LSUN
\end{minipage}
\begin{minipage}{0.49\textwidth}
\caption{
Highly realistic samples from a generator trained with BGAN on the CelebA and LSUN datasets.
These models were trained using a deep ResNet architecture with gradient norm regularization~\citep{roth2017stabilizing}.
The Imagenet model was trained on the full $1000$ label dataset without conditioning.
\label{fig:generation}}
\end{minipage}
\end{figure*}

Figure~\ref{fig:generation} shows examples from BGAN trained on these datasets.
Overall, the sample quality is very good.
Notably, our Imagenet model produces samples that are high quality, despite not being trained conditioned on the label and on the full dataset.
However, the story here may not be that BGAN necessarily generates better images than using the variational lower-bound to train the generator, since we found that images of similar quality on CelebA could be attained without the boundary-seeking loss as long as gradient norm regularization was used, rather we confirm that BGAN works well in the high-dimensional setting.

\subsection{Stability of continuous BGAN}
As mentioned above, gradient norm regularization greatly improves stability and allows for training with very large architectures.
However, training still relies on a delicate balance between the generator and discriminator: over-training the generator may destabilize learning and lead to worse results. We find that the BGAN objective is resilient to such over-training.

\paragraph{Stability in training with an overoptimized generator}
To test this, we train on the CIFAR-10 dataset using a simple DCGAN architecture.
We use the original GAN objective for the discriminator, but vary the generator loss as the variational lower-bound, the proxy loss~\citep[i.e., the generator loss function used in][]{goodfellow2014generative}, and the boundary-seeking loss (BGAN).
To better study the effect of these losses, we update the generator for $5$ steps for every discriminator step.

\begin{figure*}[t]
    \begin{minipage}{0.02\textwidth}
    \begin{turn}{90} 
    {\bf 50 epochs}
    \end{turn}
    \end{minipage}
    \begin{minipage}{0.31\textwidth}
        \centering
        {\bf GAN}
        \adjincludegraphics[width=0.99\columnwidth,trim={0 0 {.5\width} {.75\width}},clip]{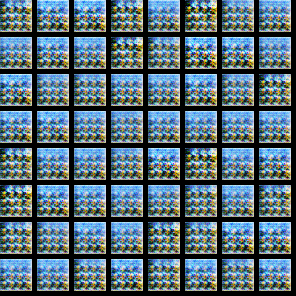}
    \end{minipage}
    \begin{minipage}{0.31\textwidth}
        \centering
        {\bf Proxy GAN}
        \adjincludegraphics[width=0.99\columnwidth,trim={0 0 {.5\width} {.75\width}},clip]{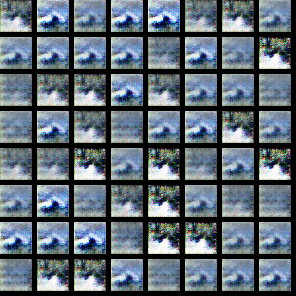}
    \end{minipage}
    \begin{minipage}{0.31\textwidth}
        \centering
        {\bf BGAN}
        \adjincludegraphics[width=0.99\columnwidth,trim={0 0 {.5\width} {.75\width}},clip]{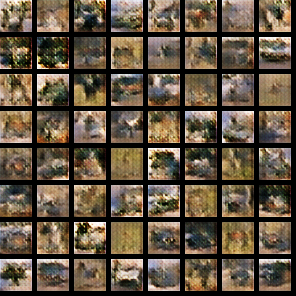}
    \end{minipage}
    
    \begin{minipage}{0.02\textwidth}
    \begin{turn}{90} 
    {\bf 100 epochs}
    \end{turn}
    \end{minipage}
    \begin{minipage}{0.31\textwidth}
        \centering
        \adjincludegraphics[width=0.99\columnwidth,trim={0 0 {.5\width} {.75\width}},clip]{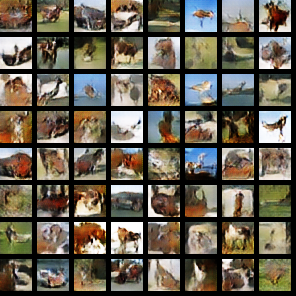}
    \end{minipage}
    \begin{minipage}{0.31\textwidth}
        \centering
        \adjincludegraphics[width=0.99\columnwidth,trim={0 0 {.5\width} {.75\width}},clip]{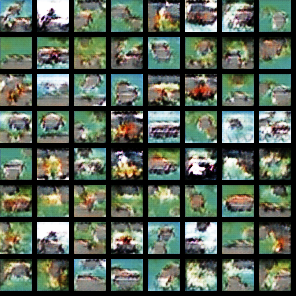}
    \end{minipage}
    \begin{minipage}{0.31\textwidth}
        \centering
        \adjincludegraphics[width=0.99\columnwidth,trim={0 0 {.5\width} {.75\width}},clip]{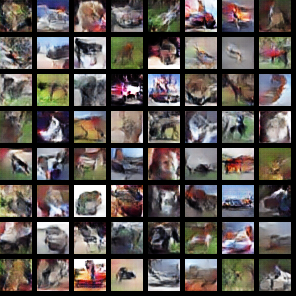}
    \end{minipage}

    \caption{
    Training a GAN with different generator loss functions and $5$ updates for the generator for every update of the discriminator.
    Over-optimizing the generator can lead to instability and poorer results depending on the generator objective function.
    Samples for GAN and GAN with the proxy loss are quite poor at $50$ discriminator epochs ($250$ generator epochs), while BGAN is noticeably better.
    At $100$ epochs, these models have improved, though are still considerably behind BGAN.}
    \label{fig:opgen}
    \vspace{-.2cm}
\end{figure*}

Our results (Figure~\ref{fig:opgen}) show that over-optimizing the generator significantly degrades sample quality.
However, in this difficult setting, BGAN learns to generate reasonable samples in fewer epochs than other objective functions, demonstrating improved stability.

\paragraph{Following the generator gradient}
We further test the different objectives by looking at the effect of gradient descent on the pixels.
In this setting, we train a DCGAN~\citep{radford2015unsupervised} using the proxy loss.
We then optimize the discriminator by training it for another $1000$ updates.
Next, we perform gradient descent directly on the pixels, the original variational lower-bound, the proxy, and the boundary seeking losses separately.

\begin{figure*}[t]
    \begin{minipage}{0.33\textwidth}
        \centering
        {\bf Starting image (generated)}
        \adjincludegraphics[width=0.99\columnwidth,trim={0 0 {.5\width} {.5\width}},clip]{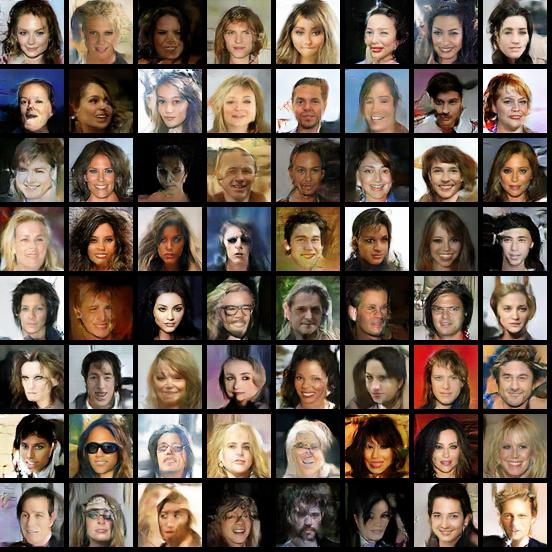}
        \vspace{.3cm}
    \end{minipage}
    \begin{minipage}{0.66\textwidth}
        \centering
        \adjincludegraphics[width=0.99\columnwidth]{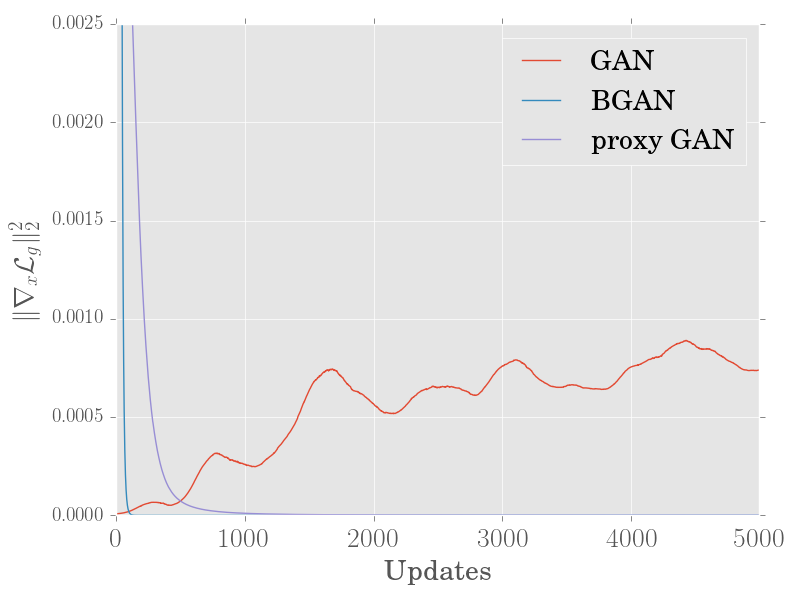}
    \end{minipage}
    
    \begin{minipage}{0.02\textwidth}
    \begin{turn}{90} 
    {\bf 10k updates}
    \end{turn}
    \end{minipage}
    \begin{minipage}{0.31\textwidth}
        \centering
        {\bf GAN}
        \adjincludegraphics[width=0.99\columnwidth,trim={0 0 {.5\width} {.5\width}},clip]{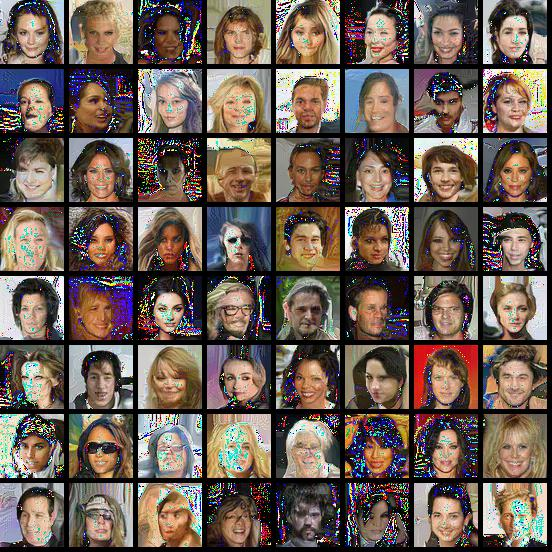}
    \end{minipage}
    \begin{minipage}{0.31\textwidth}
        \centering
        {\bf Proxy GAN}
        \adjincludegraphics[width=0.99\columnwidth,trim={0 0 {.5\width} {.5\width}},clip]{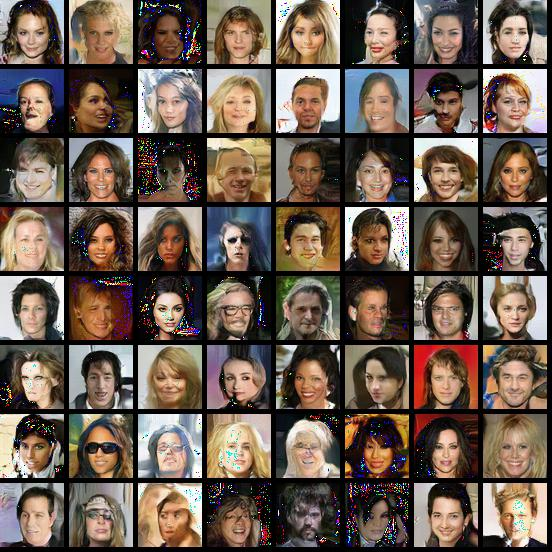}
    \end{minipage}
    \begin{minipage}{0.31\textwidth}
        \centering
        {\bf BGAN}
        \adjincludegraphics[width=0.99\columnwidth,trim={0 0 {.5\width} {.5\width}},clip]{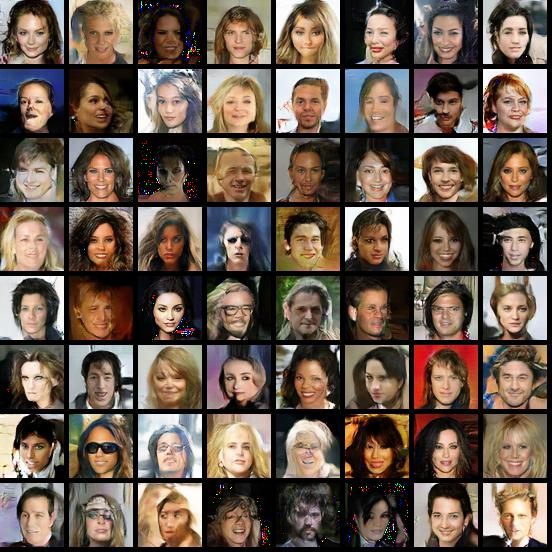}
    \end{minipage}
    
    \begin{minipage}{0.02\textwidth}
    \begin{turn}{90} 
    {\bf 20k updates}
    \end{turn}
    \end{minipage}
    \begin{minipage}{0.31\textwidth}
        \centering
        \adjincludegraphics[width=0.99\columnwidth,trim={0 0 {.5\width} {.5\width}},clip]{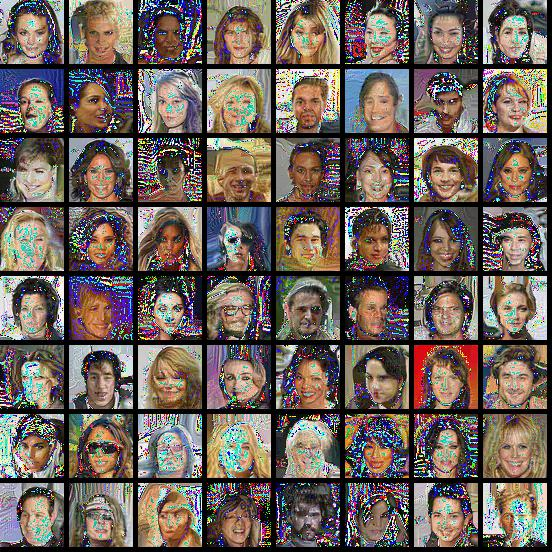}
    \end{minipage}
    \begin{minipage}{0.31\textwidth}
        \centering
        \adjincludegraphics[width=0.99\columnwidth,trim={0 0 {.5\width} {.5\width}},clip]{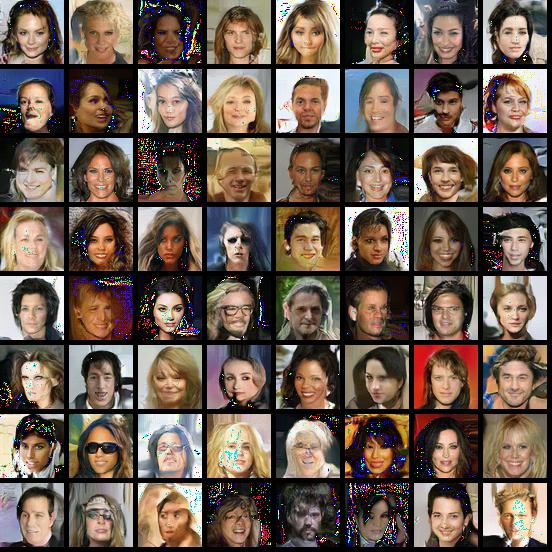}
    \end{minipage}
    \begin{minipage}{0.31\textwidth}
        \centering
        \adjincludegraphics[width=0.99\columnwidth,trim={0 0 {.5\width} {.5\width}},clip]{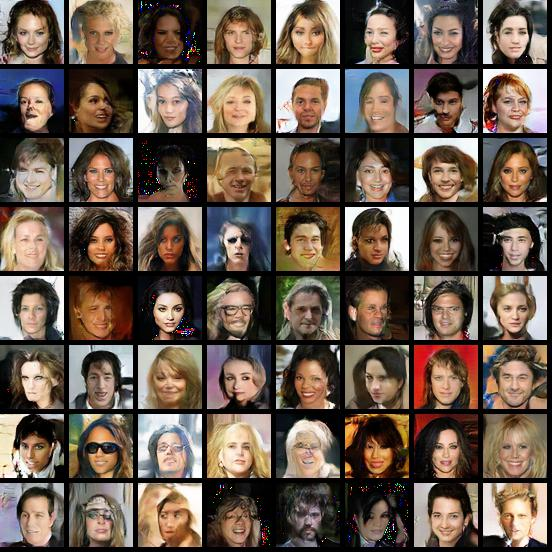}
    \end{minipage}

    \caption{
    Following the generator objective using gradient descent on the pixels.
    BGAN and the proxy have sharp initial gradients that decay to zero quickly, while the variational lower-bound objective gradient slowly increases.
    The variational lower-bound objective leads to very poor images, while the proxy and BGAN objectives are noticeably better.
    Overall, BGAN performs the best in this task, indicating that its objective will not overly disrupt adversarial learning.}
    \label{fig:opgen2}
    \vspace{-.2cm}
\end{figure*}

Our results show that following the BGAN objective at the pixel-level causes the least degradation of image quality.
This indicates that, in training, the BGAN objective is the least likely to disrupt adversarial learning.

\section{Conclusion}

Reinterpreting the generator objective to match the proposal target distribution reveals a novel learning algorithm for training a generative adversarial network~\citep[GANs, ][]{goodfellow2014generative}. 
This proposed approach of boundary-seeking provides us with a unified framework under which learning algorithms for both discrete and continuous variables are derived. Empirically, we verified our approach quantitatively and showed the effectiveness of training a GAN with the proposed learning algorithm, which we call a boundary-seeking GAN (BGAN), on both discrete and continuous variables, as well as demonstrated some properties of stability. 

\section*{Acknowledgements}
RDH thanks IVADO, MILA, UdeM, NIH grants R01EB006841 and P20GM103472, and NSF grant 1539067 for support.
APJ thanks UWaterloo, Waterloo AI lab and MILA for their support and Michael Noukhovitch, Pascal Poupart for constructive discussions.
KC thanks AdeptMind, TenCent, eBay, Google (Faculty Awards 2015, 2016), NVIDIA Corporation (NVAIL) and Facebook for their support.
YB thanks CIFAR, NSERC, IBM, Google, Facebook and Microsoft for their support.
We would like to thank Simon Sebbagh for his input and help with Theorem 2.
Finally, we wish to thank the developers of Theano~\citep{team2016theano}, Lasagne~\url{http://lasagne.readthedocs.io}, and Fuel~\citep{van2015blocks} for their valuable code-base.

\bibliography{main}
\bibliographystyle{icml2017}

\section{Appendix}
\subsection{Comparison of discrete methods}
In these experiments, we produce some quantitative measures for BGAN against WGAN with the gradient penalty~\citep[WGAN-GP,][]{gulrajani2017improved} on the discrete MNIST dataset.
In order to use back-propagation to train the generator, WGAN-GP uses the \emph{softmax probabilities} directly, bypassing the sampling process at pixel-level and problems associated with estimating gradients through discrete processes.
Despite this, WGAN-GP is been able to produce samples that visually resemble the target dataset.

Here, we train $3$ models on the discrete MNIST dataset using identical architectures with the BGAN with the JS and reverse KL $\fd$-divergences and WGAN-GP objectives.
Each model was trained for $300$ generator epochs, with the discriminator being updated $5$ times per generator update for WGAN-GP and $1$ time per generator update for the BGAN models (in other words, the generators were trained for the same number of updates).
This model selection procedure was chosen as the difference measure (i.e., JSD, reverse KL divergence, and Wasserstein distance) as estimated during training converged for each model.
WGAN-GP was trained with a gradient penalty hyper-parameter of $5.0$, which did not differ from the suggested $10.0$ in our experiments with discrete MNIST.
The BGAN models were trained with the gradient norm penalty of $5.0$~\citep{roth2017stabilizing}.

Next, for each model, we trained $3$ new discriminators with double capacity (twice as many hidden units on each layer) to maximize the the JS and reverse KL divergences and Wasserstein distance, keeping the generators fixed.
These discriminators were trained for $200$ epochs (chosen from convergence) with the same gradient-based regularizations as above.
For all of these models, the discriminators were trained using the \emph{samples}, as they would be used in practical applications.
For comparison, we also trained an additional discriminator, evaluating the WGAN-GP model above on the Wasserstein distance using the softmax probabilities.

\begin{table}[ht]
    \centering
    \caption{Estimated Jensen-Shannon and KL-divergences and Wasserstein distance by a discriminator trained to maximize the respective lowerbound (lower is better).
    Numbers are estimates averaged ovwe $12$ batches of $5000$ samples with standard devations provided in parentheses.
    All discriminators were trained using samples drawn from the softmax probabilities, with exception to an additional discriminator used to evaluate WGAN-GP where the softmax probabilities were used directly.
    In general, BGAN out-performs WGAN-GP even when comparing the Wasserstein distances.}
    \begin{tabular}{|l|c|c|c|} 
        \hline
        Train Measure   & \multicolumn{3}{c|}{Eval Measure (lower is better)} \\
        \hline
        & JS    & reverse KL    & Wasserstein \\ 
        \hline
        \hline
        BGAN - JS           & $0.37$ ($\pm 0.02$)   & $0.16$ ($\pm 0.01$)              & $0.40$ ($\pm 0.03$)  \\
        BGAN - reverse KL   & $0.44$ ($\pm 0.02$)  & $0.44$ ($\pm 0.03$)          & $0.45$ ($\pm 0.04$)  \\
        WGAN-GP (samples)   & $0.45$ ($\pm 0.03$)  & $1.32$ ($\pm 0.06$)          & $0.87$ ($\pm 0.18$)  \\
        WGAN-GP (softmax)   &-      &  -            & $0.54$ ($\pm 0.12$)  \\
        \hline
    \end{tabular}
    \label{tb:disc}
\end{table}

Final evaluation was done by estimating difference measures using $60000$ MNIST training examples againt $60000$ samples from each generator, averaged over $12$ batches of $5000$.
We used the training set as this is the distribution over which the discriminators were trained.
Test set estimates in general were close and did not diverge from training set distances, indicating the discriminators were not overfitting, but training set estimates were slightly higher on average.

Our results show that the estimates from the sampling distribution from BGAN is consistently lower than that from WGAN-GP, \emph{even when evaluating using the Wasserstein distance}.
However, when training the discriminator on the softmax probabilities, WGAN-GP has a much lower Wasserstein distance.
Despite quantitative differences, samples from these different models were indistinguishable as far as quality by visual inspection.
This indicates that, though playing the adversarial game using the softmax outputs can generate realistic-looking samples, this procedure ultimately hurts the generator's ability to model a truly discrete distribution.

\subsection{Theoretical and empirical validation of the variance reduction method}
Here we validate the policy gradient provided in Equation~\ref{eq:bgan_alpha} theoretically and empirically.
\begin{theo}
Let the expectation of the conditional KL-divergence be defined as in Equation~\ref{eq:bgan_local}. 
Then $\EE_{\prior(z)}[\DV_{KL}\left(\realcondt{x}{z} \middle\| \gencondp{x}{z} \right)] = 0 \implies \DV_{KL}(\realdensest(x) || \gendens_{\gparams}) = 0$.
\end{theo}

\begin{proof}
As the conditional KL-divergence is has an absolute minimum at zero, the expectation can only be zero when the all of the conditional KL-divergences are zero.
In other words:
\begin{align}
    \EE_{\prior(z)}[\DV_{KL}\left(\realcondt{x}{z} \middle\| \gencondp{x}{z} \right)] = 0 \implies \realcondt{x}{z} = \gencondp{x}{z}.
\end{align}
As per the definition of $\realcondt{x}{z}$, this implies that $\alpha(z) = \iw(x) = C$ is a constant.
If $\iw(x)$ is a constant, then the partition function $\beta = C \EE_{\genprobp}[1] = C$ is a constant.
Finally, when $\frac{\iw(x)}{\beta} = 1$, $\realdensest(x) = \gendens_{\gparams} \implies \DV_{KL}(\realdensest(x) || \gendens_{\gparams}) = 0$.
\end{proof}

In order to empirically evaluate the effect of using an Monte-Carlo estimate of $\beta$ from Equation~\ref{eq:bgan_grad} versus the variance-reducing method in Equation~\ref{eq:bgan_alpha}, we trained several models using various sample sizes from the prior, $\prior(z)$, and the conditional, $\gencondp{x}{z}$.

We compare both methods with $64$ samples from the prior and $5$, $10$, and $100$ samples from the conditional. 
In addition, we compare to a model that estimates $\beta$ using $640$ samples from the prior and a single sample from the conditional.
These models were all run on discrete MNIST for $50$ epochs with the same architecture as those from Section~\ref{sec:disc_results} with a gradient penalty of $1.0$, which was the minimum needed to ensure stability in nearly all the models.

Our results (Figure~\ref{fig:variance}) show a clear improvement using the variance-reducing method from Equation~\ref{eq:bgan_alpha} over estimating $\beta$.
Wall-clock times were nearly identical for methods using the same number of total samples (blue, green, and red dashed and solid line pairs).
Both methods improve as the number of conditional samples is increased.

\begin{figure*}[t]
    \centering
    \includegraphics[width=0.99\columnwidth]{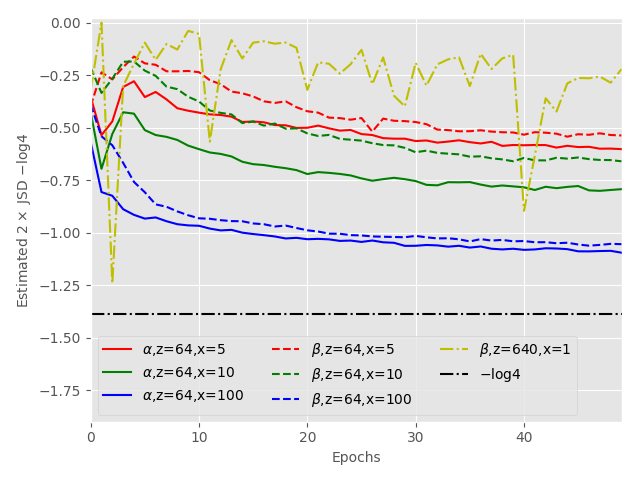}
    \caption{Comparison of the variance-reducing method from Equation~\ref{eq:bgan_alpha} and estimating $\beta$ using Monte-Carlo in Equation~\ref{eq:bgan_grad}. 
    $\alpha$ indicates the variance-reducing method, and $\beta$ is estimating $\beta$ using Monte-Carlo. 
    $z=$ indicates the number of samples from the prior, $\prior(z)$, and $x=$ indicates the number of samples from the conditional, $\gencondp{x}{z}$ used in estimation.
    Plotted are the estimated GAN distances ($2 * \text{JSD} - \log{4}$) from the discriminator.
    The minimum GAN distance, $-\log{4}$, is included for reference.
    Using the variance-reducing method gives a generator with consistently lower estimated distances than estimating $\beta$ directly.
    }
    \label{fig:variance}
\end{figure*}

\end{document}